%% file: main.tex
\documentclass{article}
\usepackage{iclr2027_conference,times}

\input{math_commands.tex}

\usepackage{amsmath}
\usepackage{amsfonts}
\usepackage{amssymb}
\usepackage{mathtools}
\usepackage{amsthm}
\usepackage{bm}
\usepackage{graphicx}
\usepackage{wrapfig}
\usepackage{float}
\usepackage{subcaption}
\usepackage{booktabs}
\usepackage{multirow}
\usepackage{enumitem}
\usepackage{xcolor}
\usepackage{colortbl}
\usepackage{url}
\usepackage{hyperref}

\newtheorem{definition}{Definition}
\newtheorem{proposition}{Proposition}

\newcommand{\cW}{\mathcal{W}}
\newcommand{\cA}{\mathcal{A}}
\newcommand{\cD}{\mathcal{D}}

\newcommand{\cS}{\mathcal{S}}

\newcommand{\skipact}{\mathrm{skip}}
\newcolumntype{L}[1]{>{\raggedright\arraybackslash}m{#1}}

\title{Learning What to Skip: Counterfactual Credit Assignment for Efficient Multi-Agent LLM Workflows}

\author{Jinfeng Xu$^{1}$, Zheyu Chen$^2$, Ziyue Peng$^3$, Zheng Lin$^4$, Shuo Yang$^5$, Jinze Li$^5$, \textbf{Zheng Xing}$^6$, \\ \textbf{Mengran Li}$^7$, \textbf{Victor C. M. Leung}$^1$\thanks{Corresponding Author.}\\ 
 \\
$^1$The University of British Columbia, Canada;
$^2$The Hong Kong Polytechnic University, Hong Kong; \\
$^3$The Hong Kong University of Science and Technology, Hong Kong; \\
$^4$University of Luxemburg, Luxemburg; 
$^5$The University of Hong Kong, Hong Kong; \\
$^6$Shenzhen University, China; 
$^7$Sun Yat-Sen University, China; 
}
\iclrfinalcopy
\begin{document}

\maketitle

\begin{abstract}
Multi-agent LLM workflows use planning, execution, verification, and summarization to improve task performance, yet the value of each component depends on the state already produced. Executing every component can waste computation or overwrite a correct intermediate answer. We formulate component omission as counterfactual credit assignment: full-workflow logs reveal the executed trajectory's reward, while controlled skip interventions reveal the consequences of omitting a future step. We introduce \textbf{Learning What to Skip} (LW2S), which learns action-specific safety models from these interventions and combines held-out calibration with domain-native guards to select skips. When an early skip is rejected, the controller can continue execution and reconsider a later component. Across mathematical reasoning, multiple-choice QA, and code generation with two instruction-model families, LW2S reduces recorded token cost while matching or improving aggregate full-workflow accuracy in the evaluated settings. Scale-up and second-topology experiments further examine component redundancy, while shared-error cases reveal why agreement alone is insufficient for skip selection. These findings connect efficient workflow execution to learning the conditional utility of individual components.
\end{abstract}

\input{Tex/Introduction}
\input{Tex/RelatedWork}
\input{Tex/Methodology}
\input{Tex/Analysis}
\input{Tex/Experiment}
\input{Tex/Conclusion}

\subsection*{AI Use Statement}

We used generative AI tools to assist with paper writing and proofreading. All claims, numerical results, scripts, and tables reported in this work were derived from the authors’ local source files and experimental outputs. We take full responsibility for the final content of this work, including the text, statements, and code.



\subsection*{Reproducibility Statement}

The experiments use open-weight instruction backbones, fixed random seeds, public benchmarks, and saved JSONL traces containing full and single-component counterfactual interventions. The repository records the trace schema, generation commands, calibration scripts, policy summaries, and resource footprints. We have released anonymized source code, trace-processing scripts, task identifiers, and generated summary tables as supplementary material.



\bibliography{iclr2027_conference}
\bibliographystyle{iclr2027_conference}

\appendix
\input{Tex/Appendix}

\end{document}

%% file: math_commands.tex
\usepackage{amsmath,amsfonts,bm}

\def\eqref#1{equation~\ref{#1}}

\def\1{\bm{1}}

\DeclareMathAlphabet{\mathsfit}{\encodingdefault}{\sfdefault}{m}{sl}
\SetMathAlphabet{\mathsfit}{bold}{\encodingdefault}{\sfdefault}{bx}{n}

\DeclareMathOperator*{\argmax}{arg\,max}

%% file: Tex/Introduction.tex
\section{Introduction}
\label{sec:introduction}

Multi-agent LLM systems increasingly solve tasks by executing structured workflows rather than a single model call. A typical workflow may plan a solution, query multiple executors, verify intermediate answers, revise the result, and summarize a final response. This design has an intuitive appeal: more specialized computation can expose errors that one model call would miss. It is also costly. Each additional component consumes tokens and latency, and each extra reasoning step creates another opportunity for a plausible but wrong intermediate output to overwrite a correct one \citep{du2023multiagent,wang2024mixture,gptswarm2024,goa2026,selforg2026,card2026}.

Recent agentic routing and graph-of-agents methods address part of this problem by deciding which agents to select or how agents should communicate \citep{wang2024mixture,du2023multiagent,gptswarm2024,graphplanner2026,goa2026,selforg2026}. These methods shift multi-agent inference away from fully dense collaboration. However, they usually optimize workflow construction from end-to-end outcomes: which model should be called, which role should be active, or which graph should be generated for a query. They do not directly answer a finer counterfactual question: \textbf{after a workflow has already produced a partial state, which future components are still worth executing?}

Our starting observation is that component utility is not a role property. A verifier can repair an incorrect solution, leave a correct solution unchanged, or damage an answer by introducing a shared misconception. A summarizer can be useful when intermediate outputs are verbose or inconsistent, but redundant when the verifier has already produced executable code or a final answer. A planner may help on difficult tasks and be pure overhead on easy ones. The same component can therefore have positive, zero, or negative marginal utility across tasks and domains. This makes fixed pruning unsafe and full execution wasteful.

We treat each optional workflow component as an intervention. Given a task \(x\), an observed workflow prefix \(h_s\), and a candidate future component \(a\), we ask whether skipping \(a\) would preserve the reward of the full workflow while reducing cost. This is a counterfactual credit assignment problem for workflow computation. Proposition~\ref{prop:nonidentifiability} formalizes the basic obstacle: full-workflow logs reveal the reward of the executed trajectory, but they do not identify the reward that would have been obtained by omitting a future component. The trace data must therefore contain controlled interventions such as omitting the verifier, omitting the summarizer, or bypassing an executor while keeping the rest of the workflow fixed.

We introduce \textbf{Learning What to Skip}, a selective decision framework for multi-agent LLM workflows. LW2S collects counterfactual traces by executing the full workflow and single-component skip interventions for the same task. It then trains action-specific safety models that estimate whether a skip action preserves full-workflow reward under the observed prefix. At deployment, LW2S accepts a skip only when held-out calibration and a domain-native guard support the action; otherwise it abstains to the full workflow or tries a later skip decision.

This formulation makes workflow computation conditional on component utility in the current state. Numerical agreement can support skipping in math reasoning, whereas option agreement in multiple-choice QA can reflect a shared misconception. In code generation, executable tests distinguish useful agreement from shared buggy logic. These differences connect the choice of skip action to domain-specific evidence, as organized in Appendix~\ref{app:failure-taxonomy}.

We evaluate LW2S across mathematical reasoning, multiple-choice QA, and code generation with two instruction-model families, larger-model scale-up diagnostics, and a second debate-refine topology. Across the main evaluated settings, LW2S reduces component cost while matching or improving aggregate full-workflow accuracy. The MBPP slice-shift analysis shows how selecting between verifier and summarizer omissions recovers savings when early skipping fails. Appendix~\ref{app:additional-details} provides the evaluation protocols, uncertainty estimates, and additional results.

Our contributions are:
\begin{itemize}[leftmargin=*]
    \item We formulate workflow efficiency as counterfactual credit assignment: full-workflow logs reveal executed rewards, but not the safety of skipping later components.
    \item We introduce LW2S, a calibrated selective controller with action-specific safety prediction, source-stable calibration, abstention, fall-through, and domain-native guards.
    \item We compare LW2S with RouteLLM, GraphPlanner, AutoMix, Prompt-LLM, and ablations under the same post-prefix skip interface across math, QA, and coding tasks.
    \item We provide model-generalization, scale-up, resource-accounting, second-topology, and failure-taxonomy evidence that separates safe redundancy from unsafe shared-consensus regimes.
\end{itemize}

%% file: Tex/RelatedWork.tex
\section{Related Work}
\label{sec:relatedwork}

\paragraph{Structured multi-agent workflows and routing.}
Multi-agent LLM systems distribute reasoning across roles, model instances, and communication structures. Debate methods ask agents to critique and revise one another. Mixture-of-agents systems aggregate multiple proposals, and graph-based systems use selected agents, directed message passing, or learned workflow graphs to avoid dense all-to-all interaction \citep{du2023multiagent,wang2024mixture,gptswarm2024,goa2026,selforg2026,card2026}. Routing methods make this efficiency objective more explicit: RouteLLM learns model-pair routing from preference data, while recent agentic routers extend routing to multi-step workflow generation and graph-structured memory \citep{routellm2024,graphplanner2026,masrouter2025}. These methods decide which model, agent, or graph should be used before or during workflow construction. LW2S takes the workflow prefix as given and asks a later question: which future component is still worth running?

\paragraph{Selective computation and verification.}
Early-exit and selective-prediction methods abstain when prediction confidence is insufficient \citep{geifman2017selective,geifman2019selectivenet,xin2020deebert}. In LLM systems, related ideas appear as confidence-based routing, self-consistency stopping, consensus heuristics, and self/meta-verification cascades such as AutoMix \citep{automix2023}. Single-agent workflow compaction offers another route by simulating homogeneous multi-agent workflows inside one model and reusing context \citep{xu2026oneflow}. These approaches exploit observable signals, but confidence, agreement, and compact execution can still conflate redundancy with correctness. Agreement is not correctness: executors can share an incorrect option on MMLU or produce structurally similar buggy code on MBPP. We therefore pair selective computation with calibrated counterfactual evidence and domain-native guards. In code domains, visible unit tests provide a useful checker signal, although verifier-scoring benchmarks show that small test suites can have limited discriminative power \citep{ficek2025scoring}.

\paragraph{Counterfactual credit for workflow components.}
Credit assignment usually attributes final reward to actions in a trajectory. Recent counterfactual credit methods for LLM collaboration restore fixed transcript prefixes to evaluate alternative agent decisions for policy optimization \citep{chen2026c3}. Agentic workflows expose a more operational inference-time question: a future component is valuable only when executing it changes reward enough to justify its cost. Full trajectories cannot answer that question, because they never reveal what would have happened had a verifier, summarizer, or later turn been omitted. LW2S records single-component interventions and turns them into skip-safety labels. Component-level computation becomes a direct learning target, rather than an implicit consequence of end-to-end reward. Appendix~\ref{app:related-boundary} summarizes the boundary between this intervention view, workflow routing, early exit, and workflow compaction.

%% file: Tex/Methodology.tex
\begin{figure*}
    \centering
    \includegraphics[width=0.95\linewidth]{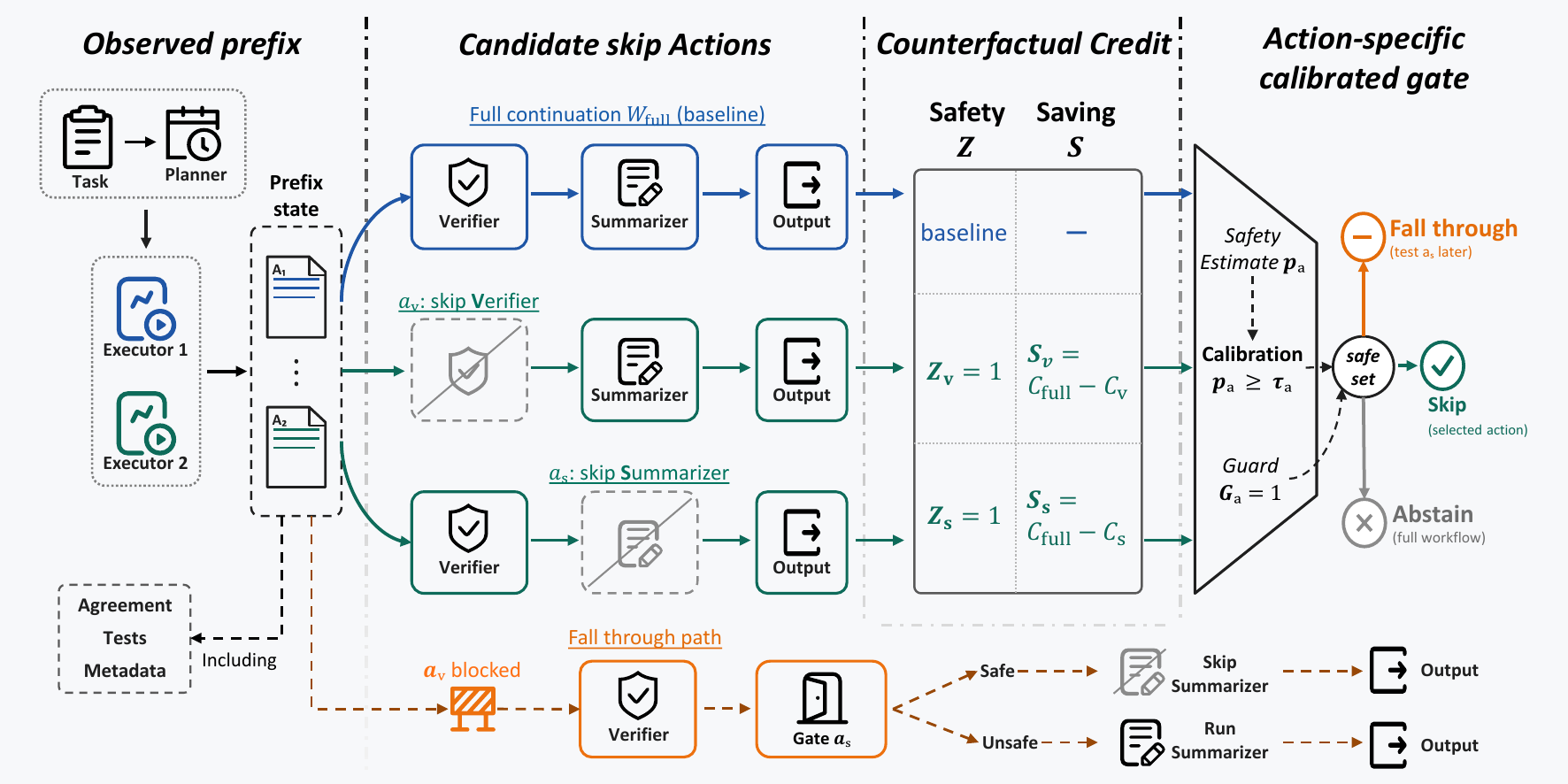}
    \vskip -0.15in
    \caption{Overview of our LW2S framework.}
    \label{fig:R2GT}
    \vskip -0.15in
\end{figure*}

\section{Methodology}
\label{sec:methodology}

LW2S learns when to omit future workflow components from controlled counterfactual traces (Figure~\ref{fig:R2GT}). It trains a separate controller over observable intermediate outputs, with recorded rewards and costs providing supervision. The base LLM remains frozen, and the controller operates through the workflow's text interface.

\subsection{Workflow Setup}
\label{sec:workflow-setup}

Let \(x\) denote a task and let a workflow \(\cW=(c_1,\ldots,c_K)\) consist of ordered components such as a planner, two executors, a verifier, and a summarizer. Executing component \(c_t\) produces an output \(o_t\) and updates the observable workflow state \(h_t=(x,o_1,\ldots,o_t)\). The full workflow produces a final answer \(y^{\mathrm{full}}\), reward \(R^{\mathrm{full}}\in[0,1]\), and cost \(C^{\mathrm{full}}\). In our experiments, the optimized cost is recorded component token cost,
\[
    C(y)=\sum_{c_t\in \mathrm{exec}(y)}
    \bigl(\mathrm{prompt\_tokens}_{t}+\mathrm{completion\_tokens}_{t}\bigr),
\]
summed over the components executed by trace \(y\). Summed component latency is reported separately for resource accounting (Appendix~\ref{app:trace-schema}); threshold selection uses token cost.

A skip action \(a\in\cA_s\) is available at state \(h_s\) when it omits one or more future components while keeping the remaining workflow semantics fixed. For example, after both executors have produced candidate solutions, the action \(\skipact\_\mathrm{verifier}\) bypasses the verifier and passes executor outputs directly to the finalization path. After the verifier has run, \(\skipact\_\mathrm{summarizer}\) bypasses final summarization and emits the verified answer.

\begin{definition}[Skip safety]
For task \(x\), state \(h_s\), and candidate skip action \(a\), let \(R^a\) and \(C^a\) be the reward and cost of the counterfactual workflow that takes action \(a\). Action is reward-preserving if
\begin{equation}
    Z_a(x,h_s) = \mathbb{I}[R^a \ge R^{\mathrm{full}}] = 1.
\end{equation}
Its realized saving is \(S_a=C^{\mathrm{full}}-C^a\). The learning problem is to select a positive-saving action with \(Z_a=1\), or abstain to the full workflow when no action is reliable.
\end{definition}

Skip safety is relative to the full workflow: preserving its reward can retain an incorrect answer when the full workflow itself fails. A skip can also improve reward by avoiding a downstream edit that damages a correct intermediate answer.

\subsection{Counterfactual Trace Collection}
\label{sec:trace-collection}

For each task, we execute the full workflow and a set of single-component interventions. The resulting trace record contains the task metadata, component outputs, executed/skipped components, final answer, reward, total cost, and latency. This produces supervised examples of the form,
\[
    (x,h_s,a,Z_a,R^a,R^{\mathrm{full}},C^a,C^{\mathrm{full}}).
\]
The key design choice is controlled intervention: only the target component is removed, so reward and cost differences are attributable to that skip action under identical task and workflow conditions.

Counterfactual trace collection precedes controller training and calibration. At deployment, the controller selects a continuation from the observed prefix. Appendix~\ref{app:intervention-protocol} specifies the intervention protocol for each skip action.

\subsection{Action-Specific Safety Models}
\label{sec:safety-models}

We train a safety model for each decision stage and skip action. For action \(a\), model estimates
\[
    p_\theta(Z_a=1\mid x,h_s,a),
\]
using text features from the task and observed component outputs, structured features extracted from intermediate answers, and task metadata. Numerical and option-letter agreement are available after the executors; verifier answer changes become available only after verification. For code tasks, executable features indicate whether the observed component outputs pass the visible MBPP unit tests. Features are constructed from the prefix available at the corresponding decision stage.

Each action has a distinct decision context. Verifier skipping depends on executor outputs, whereas summarizer skipping can use the verifier's response. Separate models associate these stage-specific features with the reward consequences of the corresponding omission.

\paragraph{Predictor implementation.}
Our implementation combines TF--IDF text features with standardized structured features in a class-balanced logistic regression model. The text branch uses lowercase unigrams and bigrams with a vocabulary capped at 1,500 features. The structured branch vectorizes prefix metadata and observable agreement or test signals. Training minimizes the regularized, class-weighted logistic loss for the binary reward-preservation label. When an action's training labels contain only one class, the implementation uses a constant predictor for that class. Each action model is fitted on training traces and then scored on calibration traces to select its acceptance threshold. Thus prediction estimates which omissions preserve reward, while calibration determines which score regions the controller will use.

\subsection{Calibration and Abstention}
\label{sec:calibration}

Let \(\cD_{\mathrm{cal}}\) be a calibration set of counterfactual traces from held-out sources such as random seeds or task slices. For each action \(a\), we sweep thresholds \(\tau\) on the predicted safety probability. A threshold is eligible only if it belongs to
\begin{equation}
\label{eq:calibration-acceptance}
    \mathcal{T}_a =
    \left\{\tau:
    \widehat{d}_a(\tau)\le d_{\max},\;
    \mathrm{UCB}_{1-\alpha}\!\left(\widehat{d}_a(\tau),n_a(\tau)\right)\le\rho,\;
    \widehat{R}_a(\tau)\ge \widehat{R}_{\mathrm{full}}
    \right\},
\end{equation}
where \(\widehat{d}_a(\tau)\) counts selected calibration examples with \(R^a<R^{\mathrm{full}}\), \(n_a(\tau)\) is the selected count, and \(\mathrm{UCB}_{1-\alpha}\) is a Wilson upper confidence bound on degradation risk, with \(\alpha\) denoting the tail probability. The final threshold is chosen from \(\mathcal{T}_a\) to maximize cost saving, with optional guard steps that move the threshold above unsafe regions of the calibration sweep. If \(\mathcal{T}_a\) is empty, the action is excluded.

With multiple calibration sources, source-stable selection requires the acceptance conditions separately on every source. Equivalently, the eligible set is the intersection of the source-specific threshold sets. This retains source-level reward losses that pooling could obscure. Section~\ref{sec:calibration-interpretation} discusses the statistical interpretation, and Appendix~\ref{app:calibration-certificates} reports the configuration and representative thresholds.

\subsection{Domain-Native Guards}
\label{sec:guards}

Calibration controls empirical skip risk, but the policy should also expose why a skip is plausible in the current state. We therefore define guard functions \(G_a(x,h_s)\in\{0,1\}\) that encode domain-native invariants. The policy may skip only when \(G_a=1\).

For math reasoning, a verifier-skip guard checks whether executors agree on the same final numeric answer. For MMLU, choice-letter consensus provides a stress test of agreement-based selection: shared misconceptions can satisfy this condition while degrading reward. For MBPP, the two-stage policy permits verifier skipping only when both executor outputs pass the visible unit tests. Its later summarizer decision uses the observed verifier state and the learned acceptance rule. Verifier test success is also evaluated as a standalone guard in Appendix~\ref{app:guard-ablation}.

A failed guard excludes the corresponding action even when the learned model assigns it high safety probability. Guards thus constrain acceptance using explicit conditions on the available workflow state.

\subsection{Selective Skip Policy}
\label{sec:selective-policy}

At a decision stage \(s\), the policy forms a safe set
\begin{equation}
\label{eq:selective-safe-set}
    \cS_s(x,h_s)=
    \left\{
    a\in\cA_s:
    p_\theta(Z_a=1\mid x,h_s,a)\ge \tau_a
    \;\wedge\;
    G_a(x,h_s)=1
    \right\}.
\end{equation}
If \(\cS_s\) is empty, the workflow executes the next component. Otherwise, the policy selects the accepted action with the largest estimated saving:
\begin{equation}
    \pi(x,h_s)=
    \argmax_{a\in\cS_s(x,h_s)} \widehat{S}_a(x,h_s).
\end{equation}

For a sequential workflow, this decision can apply more than once. In the MBPP two-stage policy, the controller first tries to skip the verifier after the executors. If the code-execution guard blocks that action, the workflow runs the verifier, then tries to skip the summarizer. This fall-through behavior is central: the policy can recover efficiency later without taking an unsafe earlier shortcut.

%% file: Tex/Analysis.tex
\section{Analysis of Counterfactual Skip Decisions}
\label{sec:analysis}

We analyze the information needed to learn skip decisions, the interpretation of empirical calibration, and the accounting of sequential fall-through.

\subsection{Why Full-Workflow Logs Are Insufficient}

The skip decision is a counterfactual question. At state \(h_s\), the full log shows the reward after executing the remaining components, but it does not show what reward would have been obtained had a future component been omitted. This creates a basic identification problem for policies trained only on full-workflow trajectories.

\begin{proposition}[Non-identifiability of skip safety from full logs]
\label{prop:nonidentifiability}
Fix a workflow prefix state \(h_s\) and a candidate skip action \(a\). Suppose the available data contain only full-workflow observations \((x,h_s,R^{\mathrm{full}},C^{\mathrm{full}})\), and suppose this distribution assigns positive probability to tasks with \(R^{\mathrm{full}}>0\). Without an intervention record or a structural assumption linking \(R^a\) to \(R^{\mathrm{full}}\), the value of
\[
    Z_a(x,h_s)=\mathbb{I}[R^a \ge R^{\mathrm{full}}]
\]
is not identifiable from the full-workflow distribution.
\end{proposition}

\begin{proof}
Consider any observed full-workflow distribution over \((x,h_s,R^{\mathrm{full}},C^{\mathrm{full}})\). Construct two counterfactual worlds that share this exact distribution. In the first world, set the unobserved skip reward to \(R^a=R^{\mathrm{full}}\) for every task, so the skip action is reward-preserving. In the second world, keep the same full-workflow observations but set \(R^a<R^{\mathrm{full}}\) on any positive-mass subset of tasks where the omitted component is necessary. The observed full logs are identical in the two worlds, yet the skip-safety labels differ on that subset. No learner that observes only the full logs can distinguish these worlds.
\end{proof}

The proposition identifies a supervision gap: observing the full continuation alone leaves the reward of an omitted component undetermined. Paired skip traces supply this missing outcome for sampled tasks and prefixes. Generalizing these labels to new prefixes remains the task of the learned controller.

\subsection{Interpretation of Empirical Calibration}
\label{sec:calibration-interpretation}

Equation~\ref{eq:calibration-acceptance} defines an empirical acceptance certificate for the selected calibration region. The Wilson bound is evaluated for each candidate threshold, while the threshold itself is selected from a grid. These pointwise bounds do not provide simultaneous coverage over the threshold search or a guarantee under distribution shift. Held-out test degradation and uncertainty are therefore reported separately from calibration results (Appendix~\ref{app:uncertainty}).

Source-stable calibration constrains the eligible thresholds across observed sources; it does not prescribe a fixed test coverage. An accepted region can cover every task in a particular test slice. The MMLU row in Table~\ref{tab:uncertainty}, for example, reports 100 selected tasks out of 100. Coverage describes how often the policy takes an action, while task-level degradation measures its observed reward consequences.

\subsection{Single-Intervention Accounting}

At a fixed prefix and threshold, the guard in Equation~\ref{eq:selective-safe-set} can only exclude candidate actions. Across a sequential workflow, however, rejecting an early action can expose a later decision. The relevant accounting unit is the complete selected continuation.

\begin{proposition}[Fall-through recovery in a sequential policy]
\label{prop:fallthrough}
Consider a two-stage workflow with candidate actions \(a_1\) at stage \(s_1\) and \(a_2\) at later stage \(s_2\). If the policy attempts \(a_1\) first and falls through to \(a_2\) whenever \(a_1\) is rejected or guard-blocked, then the set of selected actions is a subset of valid single-stage interventions \(\{a_1,a_2,\mathrm{full}\}\). The policy therefore never relies on an unobserved combined intervention, and any reward/cost assigned to a selected task corresponds to an actually collected counterfactual trace.
\end{proposition}

For each task, reward and token cost are taken from the selected trace, including the verifier cost when execution reaches the later decision. Savings from separate interventions are never added. Appendix~\ref{app:intervention-protocol} specifies the outputs and accounting for each continuation.



%% file: Tex/Experiment.tex
\section{Experiments}
\label{sec:experiment}

\subsection{Experimental Setup}
\label{sec:experimental-setup}

\paragraph{Workflows.}
The main matrix uses a fixed five-role workflow with a planner, two executors, a verifier, and a summarizer. For every task, we save the full workflow and single-component skip interventions, and train only the skip controller; no base LLM weights are updated. Appendix~\ref{app:workflow-family} describes the workflow-family scope, and Appendix~\ref{app:debate-refine-mbpp} reports a second debate-refine topology with different roles and information flow.

\paragraph{Benchmarks.}
We evaluate mathematical reasoning on MATH and GSM8K, multiple-choice QA on MMLU, and code generation on MBPP. These settings expose different skip signals: final-answer agreement in math, potentially misleading option agreement in QA, and executable unit-test evidence in code.

\paragraph{Models and resources.}
The main table uses Qwen2.5-7B-Instruct and Mistral-7B-Instruct, with Qwen2.5-14B-Instruct scale-up diagnostics in Appendix~\ref{app:qwen14-scale}. The saving metric is recorded component-token reduction relative to the full trace for the same task. Resource accounting, including trace counts, token totals, latency, hardware, and API use, is reported in Appendix~\ref{app:trace-schema}.

\paragraph{Baselines and policies.}
We compare against RouteLLM \citep{routellm2024}, GraphPlanner \citep{graphplanner2026}, AutoMix \citep{automix2023}, and Prompt-LLM under the same post-prefix skip protocol. Each policy observes the same workflow prefix and either accepts a skip action or falls back to the full workflow. Fixed skip, random skip, consensus-only, and guard-only variants serve as diagnostic ablations. Appendix~\ref{app:baseline-protocol} gives the baseline specification and clarifies how routing and adaptive-computation principles are instantiated for component omission.

\input{Tab/MainEvidence}

\paragraph{Comparison controls.}
All policies use the same saved workflow traces, data partition, reward parser, and token accounting. At a decision stage, inputs are restricted to the observed prefix. RouteLLM and GraphPlanner also receive intervention-derived binary skip-safety labels; LW2S organizes prediction and calibration by action and adds source checks, guards, and sequential fall-through. Table~\ref{tab:baseline-spec} specifies the inputs, supervision, and deployment rule of each method.

\subsection{Main Results}
\label{sec:main-results}

\begin{wrapfigure}[13]{r}{0.43\textwidth}
\centering
\vskip -0.04in
\includegraphics[width=0.98\linewidth]{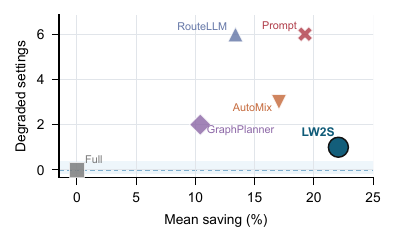}
\vskip -0.1in
\caption{Qwen2.5-7B method summary; lower right is better.}
\label{fig:risk-saving-pareto}
\vskip -0.15in
\end{wrapfigure}

LW2S reduces token cost in all six Qwen2.5-7B settings, with zero observed task-level degradations in five settings and one degradation on MMLU (Table~\ref{tab:main-evidence}). Under the shared skip protocol, RouteLLM, GraphPlanner, AutoMix, and Prompt-LLM also save tokens, and each incurs degradation in at least one setting. Figure~\ref{fig:risk-saving-pareto} summarizes this trade-off across settings; Appendix~\ref{app:safety-efficiency-ledger} provides the corresponding LW2S uncertainty estimates and full-workflow references.

On MATH, LW2S improves verifier-setting accuracy from 65.2\% to 66.4\% with 25.1\% saving and omits summarization with 18.4\% saving. It preserves full-workflow reward on GSM8K. On MMLU, aggregate accuracy remains 73.0\% with 8.6\% saving, although one task degrades. Figure~\ref{fig:main-diagnostics}(a) reports selected counts and Wilson degradation-risk bounds alongside these results. Appendix~\ref{app:uncertainty} gives the corresponding sample sizes and uncertainty estimates.

The MBPP split is the main source-shift test. On Qwen2.5-7B slice0, LW2S improves pass rate from 69.0\% to 70.0\% while saving 28.5\%. The strongest named-baseline saving reaches 24.7\% but drops pass rate to 67.0\%. On slice1, LW2S requires both learned acceptance and passing executor tests to skip the verifier. If either condition fails, it executes the verifier and evaluates the later summarizer skip. This preserves the full 63.0\% pass rate while saving 31.2\%, whereas the most aggressive Prompt-LLM baseline saves 27.0\% but drops to 61.0\%.

The same pattern survives a cross-family check with Mistral-7B. LW2S has zero observed degradations across all six settings, preserves or improves full-workflow reward, and improves the two MBPP slices by four and six pass-rate points. The competing rows are less stable: Prompt-LLM saves aggressively but degrades every setting, while AutoMix is stronger in several rows without reaching uniform zero-degradation behavior. Appendix~\ref{app:qwen14-scale} adds a Qwen2.5-14B scale-up slice, where MATH and MBPP retain skip-safe redundancy but MMLU remains a high-saving unsafe boundary. The debate-refine diagnostic in Appendix~\ref{app:debate-refine-mbpp} checks topology shift; LW2S reaches 68.0\% pass rate versus 67.0\% for the full workflow while saving 16.65\% recorded token cost.

\subsection{Attributing the Gain to Component-Level Credit}
\label{sec:component-credit}

\begin{figure}[!t]
\centering
\includegraphics[width=0.98\linewidth]{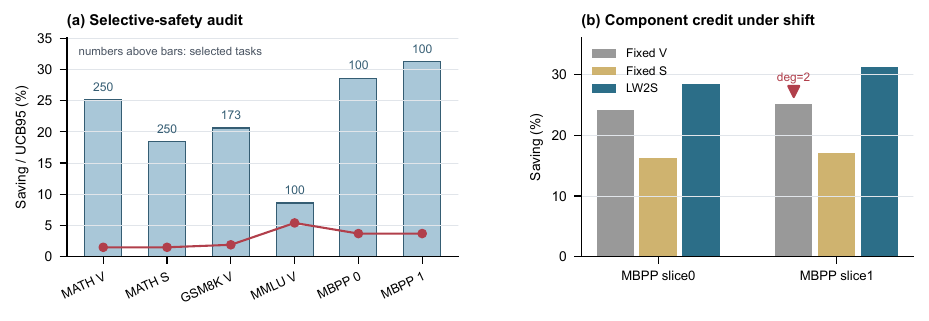}
\vskip -0.1in
\caption{Diagnostics. (a) LW2S coverage, saving, and Wilson degradation-risk bounds on Qwen2.5-7B. (b) MBPP component credit under slice shift; LW2S reallocates verifier and summarizer skips while preserving reward.}
\vskip -0.15in
\label{fig:main-diagnostics}
\end{figure}

Figure~\ref{fig:main-diagnostics}(b) examines component selection on MBPP. On slice1, fixed verifier skipping saves tokens but loses two pass-rate points. LW2S selects verifier skipping on 57 tasks and summarizer skipping after verification on the remaining selected tasks, preserving the full pass rate with the highest recorded saving. The comparison shows how access to a later omission recovers efficiency when early skipping fails; Appendix~\ref{app:component-credit-diagnostic} reports the baseline results and available action counts.

\subsection{Code Generation as a Deterministic-Checker Domain}
\label{sec:mbpp-results}

MBPP tests whether the framework transfers beyond math reasoning into a domain with executable correctness checks. On three initial 50-task seeds, summarizer skipping preserves full pass rate with zero degradations, while verifier skipping improves pass rate by three tasks in aggregate. Two disjoint unseen 100-task slices then separate safe verifier redundancy from unsafe shared-code errors.

The first slice supports verifier skipping: guarded verifier removal improves pass rate from 69\% to 70\% with zero degradations and 15.75\% saving. The second slice exposes a boundary. Unguarded verifier skipping drops pass rate from 63\% to 61\%, with two degradations where both executors produced the same incorrect code and the verifier repaired it. The code-execution guard blocks these failures and restores the full pass rate, but saves only 12.89\% when used as a single-stage verifier policy. On this slice, the two-stage LW2S policy preserves the same pass rate with 31.2\% saving by combining accepted verifier omissions with later summarizer omissions. This comparison demonstrates the efficiency recovered through later decisions when early skipping is rejected.

\subsection{What the Boundary Cases Show}
\label{sec:boundary}

The boundary cases identify the method's operating principle. On MMLU, option-letter agreement can reflect shared misconception. Appendix~\ref{app:mmlu-failure-cases} reports a case where both executors choose the same wrong option and the verifier repairs the answer. On MBPP, abstract-syntax equality between executor programs can reflect shared buggy logic. Appendix~\ref{app:guard-ablation} shows that AST equality alone selects 94 of 100 slice1 tasks and incurs two degradations, while executable evidence blocks those failures. Appendix~\ref{app:failure-taxonomy} organizes these cases together with redundant late summarization. In each regime, LW2S treats disagreement between calibration, guards, and observed state as a reason to abstain or fall through to a later decision.

\subsection{Component Utility under Topology Shift}
\label{sec:topology-results}

The debate-refine workflow replaces the main template's executor interaction with opposing code proposals followed by a critic and finalizer. On the 100-task MBPP slice, LW2S reaches 68.0\% pass rate with 16.65\% saving, compared with 67.0\% for full execution. Fixed summarizer skipping reaches the same operating point, while fixed verifier skipping preserves 67.0\% pass rate with 27.42\% saving. All three have zero observed degradations (Table~\ref{tab:debate-refine-mbpp}). These results identify a topology where a fixed late omission already captures the realized LW2S benefit. In contrast, the main MBPP slice1 requires retaining the verifier on part of the slice and recovering savings through later summarizer decisions. Together, the two workflows distinguish broadly redundant components from state-dependent omission opportunities.



%% file: Tab/MainEvidence.tex
\begin{table}[!t]
\centering
\caption{Main same-protocol evidence matrix. Acc. and Save are percentages; Acc. superscripts count task-level degradations against the full workflow. No superscript means zero observed degradation. Mistral MMLU averages three rotations.}
\vskip -0.1in
\label{tab:main-evidence}
\small
\setlength{\tabcolsep}{2.6pt}
\renewcommand{\arraystretch}{0.95}
\resizebox{\textwidth}{!}{%
\begin{tabular}{@{}llcrrrrrr@{}}
\toprule
\textbf{Backbone} & \textbf{Setting} & \textbf{Metric} & \textbf{Full} & \textbf{RouteLLM} & \textbf{GraphPlanner} & \textbf{AutoMix} & \textbf{Prompt-LLM} & \textbf{LW2S} \\
\midrule
\multirow{12}{*}{Qwen2.5-7B} & \multirow{2}{*}{MATH verifier} & Acc. & 65.2 & 64.8$^{1}$ & 65.2$^{1}$ & 65.6 & 64.8$^{1}$ & \textbf{66.4} \\
& & Save & 0.0 & 14.6 & 11.2 & 18.3 & 20.5 & \textbf{25.1} \\
\addlinespace[0.2pt]
& \multirow{2}{*}{MATH summarizer} & Acc. & 65.2 & 64.8$^{1}$ & 65.2 & 65.2 & 64.8$^{1}$ & \textbf{65.2} \\
& & Save & 0.0 & 11.8 & 9.5 & 14.2 & 16.5 & \textbf{18.4} \\
\addlinespace[0.2pt]
& \multirow{2}{*}{GSM8K verifier} & Acc. & 88.5 & 88.0$^{1}$ & 88.5 & 88.0$^{1}$ & 88.0$^{1}$ & \textbf{88.5} \\
& & Save & 0.0 & 12.7 & 8.9 & 15.8 & 19.1 & \textbf{20.6} \\
\addlinespace[0.2pt]
& \multirow{2}{*}{MMLU verifier} & Acc. & 73.0 & 73.0$^{1}$ & 73.0$^{1}$ & 73.0$^{1}$ & 72.0$^{1}$ & \textbf{73.0}$^{1}$ \\
& & Save & 0.0 & 5.8 & 4.3 & 7.0 & 7.7 & \textbf{8.6} \\
\addlinespace[0.2pt]
& \multirow{2}{*}{MBPP slice0} & Acc. & 69.0 & 68.0$^{1}$ & 69.0 & 69.0 & 67.0$^{2}$ & \textbf{70.0} \\
& & Save & 0.0 & 16.9 & 13.4 & 22.1 & 24.7 & \textbf{28.5} \\
\addlinespace[0.2pt]
& \multirow{2}{*}{MBPP slice1} & Acc. & 63.0 & 62.0$^{1}$ & 63.0 & 62.0$^{1}$ & 61.0$^{2}$ & \textbf{63.0} \\
& & Save & 0.0 & 18.5 & 15.2 & 24.9 & 27.0 & \textbf{31.2} \\
\midrule
\multirow{12}{*}{Mistral-7B} & \multirow{2}{*}{MATH verifier} & Acc. & 12.8 & 12.4$^{1}$ & 12.8 & 12.6$^{1}$ & 12.2$^{1}$ & \textbf{12.8} \\
& & Save & 0.0 & 11.3 & 8.8 & 14.6 & 17.0 & \textbf{19.1} \\
\addlinespace[0.2pt]
& \multirow{2}{*}{MATH summarizer} & Acc. & 12.8 & 12.8 & 13.0 & 13.2 & 12.4$^{1}$ & \textbf{14.0} \\
& & Save & 0.0 & 12.9 & 10.4 & 16.3 & 18.6 & \textbf{21.8} \\
\addlinespace[0.2pt]
& \multirow{2}{*}{GSM8K verifier} & Acc. & 39.5 & 38.7$^{1}$ & 39.1$^{1}$ & 39.0$^{1}$ & 38.3$^{1}$ & \textbf{39.5} \\
& & Save & 0.0 & 13.6 & 10.7 & 17.5 & 20.1 & \textbf{22.2} \\
\addlinespace[0.2pt]
& \multirow{2}{*}{MMLU verifier} & Acc. & 55.0 & 54.6$^{1}$ & 55.4 & 55.8 & 54.2$^{1}$ & \textbf{57.2} \\
& & Save & 0.0 & 15.7 & 12.1 & 19.2 & 22.4 & \textbf{25.3} \\
\addlinespace[0.2pt]
& \multirow{2}{*}{MBPP slice0} & Acc. & 28.0 & 28.0 & 29.0 & 30.0 & 27.0$^{1}$ & \textbf{32.0} \\
& & Save & 0.0 & 14.1 & 11.5 & 18.3 & 21.0 & \textbf{23.7} \\
\addlinespace[0.2pt]
& \multirow{2}{*}{MBPP slice1} & Acc. & 25.0 & 25.0 & 26.0 & 28.0 & 24.0$^{1}$ & \textbf{31.0} \\
& & Save & 0.0 & 14.5 & 11.8 & 18.6 & 21.4 & \textbf{24.1} \\
\bottomrule
\end{tabular}
}
\end{table}

%% file: Tex/Conclusion.tex
\section{Conclusion}




We introduced LW2S, which learns when to omit future workflow components through counterfactual supervision, action-specific calibration, and domain-native guards. Controlled interventions expose the reward consequences of a skipped step, providing a learning signal absent from full-workflow logs alone. Across the evaluated reasoning and coding tasks, LW2S reduces token cost while matching or improving aggregate full-workflow accuracy. The component and topology analyses distinguish broadly redundant calls from omissions that depend on the current state. Shared-error cases show why agreement alone is insufficient, while sequential fall-through recovers savings when an earlier skip is rejected. These findings motivate learning the conditional utility of workflow components as a basis for adaptive multi-agent execution.

%% file: Tex/Appendix.tex
\section{Protocols, Diagnostics, and Additional Evidence}
\label{app:additional-details}

We detail trace collection, intervention protocols, baseline implementations, and calibration, followed by scale-up, topology, and failure analyses. Throughout, \emph{accuracy} reports aggregate task reward and \emph{Deg.} counts individual tasks on which a policy underperforms the full workflow. Improvements on some tasks can offset degradations on others, so we report both quantities alongside token savings.

\subsection{Trace Schema, Splits, and Reproducibility}
\label{app:trace-schema}

\paragraph{Trace schema.}
Each JSONL record corresponds to one benchmark task and contains the task identifier, benchmark split, prompt, gold target, full-workflow component outputs, single-component skip-intervention outputs, parsed rewards, component token counts, component latency, and structured state metadata. The metadata includes executor agreement, verifier decision fields, code-execution outcomes, and guard values. This design makes the unit of evidence a task-level counterfactual bundle rather than a single model response: for the same input, we can compare the full workflow against controlled omissions of future components.

\paragraph{Split discipline.}
Thresholds are selected on calibration traces before evaluation on a held-out trace source. With multiple calibration sources, a threshold must satisfy the degradation criterion separately on each source. If none qualifies, LW2S retains the full workflow or, in the MBPP two-stage setting, falls through to the later summarizer decision. Test intervention rewards are used to evaluate the selected policy.

\paragraph{Resource ledger.}
Table~\ref{tab:resource-accounting} reports the trace-generation resource ledger used by the paper. The table separates the main fixed-template matrix, the Mistral cross-family matrix, the Qwen2.5-14B scale-up diagnostics, and the second debate-refine workflow. This separation matters because the second workflow is used as a topology diagnostic rather than as part of the original main-matrix evidence.

\input{Tab/ResourceAccounting}

The accounting unit is a saved full/intervention trace. Token cost is computed from recorded prompt and completion tokens for the components executed in that trace. Reported latency is summed component latency from the same records. It is included to make the computational footprint auditable, but it is not used to tune thresholds or select policies. We do not reconstruct end-to-end wall-clock time from these logs because runs can be interrupted, resumed, and scheduled differently; using recorded component latency gives a stable per-trace resource diagnostic. All reported traces use open-weight models on one NVIDIA RTX PRO 6000 Blackwell Server Edition, and no commercial API calls are used.

\subsection{Workflow-Family Coverage}
\label{app:workflow-family}

The main matrix fixes the planner--executor--executor--verifier--summarizer template to make intervention labels comparable across benchmarks. The debate-refine workflow changes agent roles and information flow, testing whether component-level skip opportunities persist under a different topology. Table~\ref{tab:workflow-family} lists the templates, backbones, and tasks covered by these evaluations.

\input{Tab/WorkflowFamilyMain}

\subsection{Boundary to Adjacent Paradigms}
\label{app:related-boundary}

Table~\ref{tab:related-boundary} positions LW2S relative to routing, early-exit, counterfactual-credit, and workflow-compaction methods. The key distinction is the decision target. A model router selects which model to call; a workflow router chooses a graph or agent allocation; an early-exit method rejects or emits a prediction based on confidence. LW2S instead observes a partially executed multi-agent workflow and asks whether a specific future component is still needed for that task.

\input{Tab/RelatedBoundary}

Table~\ref{tab:related-boundary} compares the decision targets and supervision of these paradigms. LW2S learns the reward consequences of individual omissions. Its sequential policy (Section~\ref{sec:selective-policy}) can retain an early component and reconsider a later one, allowing the omission to depend on the evolving workflow state.

\subsection{Intervention Protocol}
\label{app:intervention-protocol}

The intervention protocol in Table~\ref{tab:intervention-protocol} fixes the workflow template, prompts, decoding settings, answer parsers, and task family while changing exactly one future component. The observed prefix is the state available to the controller at test time. For verifier skipping, the controller has seen both executor outputs but has not run the verifier. For summarizer skipping, it has seen the verifier output and decides whether the summarizer is still needed. This design prevents leakage from components that would not have been available at the decision point.

\input{Tab/InterventionProtocol}

The MBPP two-stage policy applies the learned threshold and code guard jointly at the verifier decision. Rejection by either condition leads to verifier execution and a later summarizer decision. Each task is assigned the reward and token cost of exactly one recorded continuation: skip verifier, skip summarizer, or full workflow. Accordingly, the number satisfying a standalone guard and the number selected by LW2S measure different predicates, even on identical tasks and traces.

\subsection{Baseline Evaluation Protocol}
\label{app:baseline-protocol}

The named baselines share LW2S's saved workflow traces, train/calibration/test partition, reward parser, and cost accounting. At test time, each decision rule receives the observed prefix and its available tool outputs; intervention rewards remain evaluation targets. This protocol holds workflow execution fixed across decision rules.

\input{Tab/BaselineSpec}

Table~\ref{tab:baseline-spec} specifies the implementations evaluated through the shared component-omission interface. RouteLLM \citep{routellm2024} uses observed prefix features for a binary skip decision, and GraphPlanner \citep{graphplanner2026} augments the prefix with memory features. Both receive binary skip-safety supervision derived from the intervention traces. AutoMix \citep{automix2023} uses trace reward and cost for cost-sensitive selection, while Prompt-LLM asks the backbone to judge a candidate skip in context. LW2S organizes prediction and calibration by action and combines source checks with guards and sequential fall-through. Thus the comparison varies the decision rule and its use of supervision; intervention-derived labels are shared by several methods. Fixed skip, random skip, consensus-only, and guard-only variants serve as diagnostic ablations.

\subsection{Qwen2.5-14B Scale-Up Diagnostics}
\label{app:qwen14-scale}

The Qwen2.5-14B experiment evaluates fixed and consensus-based omissions on MATH seed55, MMLU seed2, and MBPP slice1 (Table~\ref{tab:qwen14-scale}). These diagnostic policies test how component utility changes with backbone scale. Verifier and summarizer redundancy persists on MATH, and MBPP retains code-generation skip opportunities. On MBPP, however, summarizer removal incurs an individual degradation despite increasing aggregate pass rate.

\input{Tab/ScaleUpQwen14}

On MMLU, fixed verifier skipping raises accuracy from 82.0\% to 83.0\%, and executor-consensus skipping raises it to 84.0\%. Both also cause individual degradations: gains on other tasks offset the loss of answers repaired by the verifier. Shared-error failures therefore persist with the larger backbone, even where average accuracy and token savings improve.

\subsection{Second-Workflow Debate-Refine Diagnostic}
\label{app:debate-refine-mbpp}

To examine component utility under topology shift, we evaluate a debate-refine workflow on MBPP. A framer initializes the problem, two agents produce opposing code proposals, a critic/verifier inspects the candidates, and a finalizer emits the answer. This changes the agent roles and information flow while retaining the verifier and summarizer intervention interfaces used by LW2S.

Section~\ref{sec:topology-results} compares the resulting LW2S operating point with fixed verifier and summarizer omissions. Table~\ref{tab:debate-refine-mbpp} additionally reports the two decision stages, hindsight selection, and consensus- and code-guard policies.

Table~\ref{tab:debate-refine-mbpp} reports the parsed result for this second workflow. The full debate-refine workflow reaches 67.0\% pass rate on the 100-task MBPP slice. The LW2S rows improve aggregate pass rate to 68.0\% while saving 16.65\% of token cost. The hindsight diagnostic indicates that additional reward-improving omissions exist in the recorded interventions, while fixed and guard-only rows show that this evaluated slice also contains simple safe omissions. All listed rows have zero observed task-level degradations on this evaluated slice. The result serves a specific generalization purpose: the skip-credit signal is not confined to the original fixed-template workflow.

\input{Tab/DebateRefineMBPP}

\subsection{Unified Safety-Efficiency Ledger}
\label{app:safety-efficiency-ledger}

Table~\ref{tab:safety-efficiency-ledger} consolidates LW2S accuracy, task-level degradation, Wilson upper bounds, and token savings across backbones and workflows. The full-workflow rows provide the accuracy and cost references for each comparison.

\input{Tab/SafetyEfficiencyLedger}

The Qwen2.5-7B entries complement the method-level comparison in Figure~\ref{fig:risk-saving-pareto}. The additional rows report cross-family and topology results under the same accuracy, degradation, and cost definitions.

\subsection{Calibration Certificates and Uncertainty}
\label{app:calibration-certificates}

Calibration follows Equation~\ref{eq:calibration-acceptance}, with \(\alpha=0.2\), \(d_{\max}=0\), \(\rho=0.15\), and a threshold grid in increments of 0.05. Source-stable variants apply the acceptance conditions separately to each calibration source. Table~\ref{tab:calibration-certificate} reports representative accepted regions. Section~\ref{sec:calibration-interpretation} explains the statistical interpretation of this threshold search.

\input{Tab/CalibrationCertificate}

The columns list calibration size, selected threshold, observed degradation, UCB$_{80}$, and calibration saving. Threshold selection uses these calibration measurements before held-out test evaluation.

Table~\ref{tab:uncertainty} complements the calibration view with test-set coverage and uncertainty diagnostics. The selected column counts how many tasks were routed away from the full workflow by LW2S. Five settings have zero observed degradations, for which the Wilson upper bound depends on sample size. MMLU has one degradation among 100 tasks, yielding a 5.4\% upper bound, compared with 3.7\% for the zero-degradation 100-task MBPP slices.

\subsection{LW2S Coverage and Degradation-Risk Diagnostics}
\label{app:uncertainty}

\input{Tab/Uncertainty}

Coverage and saving measure different aspects of the policy. Coverage counts tasks assigned a skip, whereas saving also depends on the token cost of the omitted continuation. Thus similar selected counts can yield different savings across settings. The MMLU entry reports full test coverage (100/100); source-wise calibration does not imply low coverage on this test slice. The degradation count and Wilson bound describe the corresponding observed reward losses and finite-sample uncertainty.

\subsection{Component-Level Credit Diagnostic}
\label{app:component-credit-diagnostic}

Figure~\ref{fig:main-diagnostics}(b) and Table~\ref{tab:component-discriminative} report the MBPP component-selection diagnostic. The LW2S rows match Table~\ref{tab:main-evidence}: slice0 selects 69 verifier skips and 31 summarizer skips, with 70.0\% accuracy and 28.5\% saving; slice1 selects 57 and 43, respectively, with 63.0\% accuracy and 31.2\% saving. Both have zero observed task-level degradations. These action counts show how the policy allocates omissions across the two decision stages.

\input{Tab/ComponentDiscriminative}

\subsection{Failure and Abstention Taxonomy}
\label{app:failure-taxonomy}

Table~\ref{tab:failure-taxonomy} groups the boundary cases into three mechanism-level regimes: shared wrong consensus, superficial code agreement, and redundant late summarization. The taxonomy explains why LW2S sometimes skips aggressively, sometimes abstains, and sometimes delays the skip decision to a later workflow state.

\input{Tab/FailureTaxonomy}

The first two regimes expose failures of agreement-based selection. MMLU executors can share a wrong option, and MBPP executors can share a bug despite AST-level similarity. Redundant late summarization provides a complementary opportunity: a verifier or checked intermediate output may already contain the final answer. Together, these regimes explain why component utility depends on the observed prefix and why a blocked early omission can leave a useful later skip available.

\subsection{Guard Ablations}
\label{app:guard-ablation}

Table~\ref{tab:mbpp-guard-ablation} compares prefix-available guard conditions and an offline structural diagnostic on two unseen 100-task MBPP slices. Executor test outcomes and executor AST equality are available before verification; verifier test outcomes become available after verification. Verifier/summarizer AST equality requires the completed summarizer output and is evaluated in hindsight to characterize unchanged code. It is excluded from online summarizer-skip decisions.

\input{Tab/MBPPGuardAblation}

AST equality is the clearest failure mode. On slice0 it appears safe, selecting 93 tasks with no degradation. On slice1 it selects 94 tasks but incurs two degradations, because both executors can share the same syntactic or structural bug. Guards that require executable evidence, such as both executors passing visible tests, are more conservative but block those slice1 failures. This supports the main design principle: guards should encode domain-native invariants rather than superficial agreement signals.

\subsection{MMLU Failure Case}
\label{app:mmlu-failure-cases}

Table~\ref{tab:mmlu-failure-cases} reports the MMLU failure case: both executors select option A, while the verifier corrects the answer to C. Skipping verification retains the shared error.

\input{Tab/MMLUFailureCases}

Answer agreement is compatible with both redundant verification and a shared error that verification repairs. Intervention labels distinguish these outcomes during training, while source-specific calibration determines which predicted skip regions are accepted. This case illustrates the feature ambiguity behind agreement-based selection and the observed MMLU degradation despite unchanged aggregate accuracy.

%% file: Tab/ResourceAccounting.tex
\begin{table}[!h]
\centering
\caption{Resource accounting for reported open-weight traces. Five-role denotes planner, two executors, verifier, and summarizer; Debate denotes debate-refine. Latency is summed component latency and is not used for threshold selection; no commercial API calls are used.}
\vskip -0.1in
\label{tab:resource-accounting}
\scriptsize
\renewcommand{\arraystretch}{1.08}
\setlength{\tabcolsep}{2.6pt}
\begin{tabular}{@{}L{0.16\textwidth}L{0.12\textwidth}L{0.09\textwidth}rrrrL{0.10\textwidth}L{0.04\textwidth}@{}}
\toprule
Block & Backbone & Workflow & Tasks & Traces & Tokens (M) & Latency (h) & GPU & API \\
\midrule
Main matrix & Qwen2.5-7B & Five-role & 1100 & 6600 & 12.74 & 9.02 & PRO6000 & none \\
Model generalization & Mistral-7B & Five-role & 1490 & 8940 & 20.31 & 14.52 & PRO6000 & none \\
Scale-up diagnostic & Qwen2.5-14B & Five-role & 250 & 1500 & 3.27 & 4.62 & PRO6000 & none \\
Second workflow & Qwen2.5-7B & Debate & 200 & 1200 & 3.14 & 2.90 & PRO6000 & none \\
Second workflow & Qwen2.5-7B & Debate & 250 & 1500 & 3.83 & 2.97 & PRO6000 & none \\
\midrule
\textbf{Total} & -- & -- & 3290 & 19740 & 43.30 & 34.04 & PRO6000 & none \\
\bottomrule
\end{tabular}
\vskip -0.1in
\end{table}

%% file: Tab/WorkflowFamilyMain.tex
\begin{table}[!t]
\centering
\caption{Workflow-family coverage. The main matrix uses a controlled five-role template; scale-up and debate-refine rows test persistence beyond one backbone and one topology.}
\vskip -0.1in
\label{tab:workflow-family}
\scriptsize
\setlength{\tabcolsep}{3.5pt}
\begin{tabular}{@{}L{0.20\linewidth}L{0.21\linewidth}L{0.17\linewidth}L{0.34\linewidth}@{}}
\toprule
Workflow family & Benchmarks & Backbones & Evidence role \\
\midrule
Five-role & MATH, GSM8K, MMLU, MBPP & Qwen2.5-7B, Mistral-7B & Main safety--efficiency matrix with same-protocol named baselines and calibrated skip policies. \\
Scale-up fixed template & MATH, MMLU, MBPP & Qwen2.5-14B & Larger-backbone diagnostic for skip-safe redundancy and unsafe high-saving shortcuts. \\
Debate-refine & MBPP & Qwen2.5-7B & Topology diagnostic with different roles and information flow from the main workflow. \\
\bottomrule
\end{tabular}
\vskip -0.1in
\end{table}

%% file: Tab/RelatedBoundary.tex
\begin{table}[!t]
\centering
\caption{Technical boundary between LW2S and adjacent efficiency paradigms. The key distinction is component-level skip supervision from controlled interventions after a workflow prefix has been observed.}
\vskip -0.1in
\label{tab:related-boundary}
\scriptsize
\setlength{\tabcolsep}{3pt}
\renewcommand{\arraystretch}{1.08}
\begin{tabular}{@{}L{0.17\textwidth}L{0.23\textwidth}L{0.22\textwidth}L{0.31\textwidth}@{}}
\toprule
Method family & Decision target & Supervision or evidence & Boundary to LW2S \\
\midrule
Model routers & choose a model before execution & preference or task reward & Optimize model/cost trade-off; no component-level skip intervention labels or sequential fall-through. \\
Graph/workflow routers & choose a graph or agent allocation & end-to-end reward & Estimate workflow utility; usually does not identify whether a specific future component is safely omissible. \\
Selective prediction and early exit & emit, reject, or exit at a confidence gate & labels plus reject loss & Controls prediction risk--coverage; does not intervene on verifier/summarizer components in an observed workflow prefix. \\
Counterfactual MARL credit & assign action advantage during policy learning & counterfactual returns & Uses counterfactual credit, but it is not a calibrated deployment-time skip policy over LLM workflow components. \\
Workflow compaction & replace an evaluated workflow before deployment & workflow search or evaluation & Finds fewer turns or a single-agent simulation; no per-task fall-through after a blocked component skip. \\
\textbf{LW2S} & omit a specific future component after observing a prefix & full and single-skip traces & Learns action-specific skip labels, calibrates source-stable risk, and can fall through to later skip decisions. \\
\bottomrule
\end{tabular}
\vskip -0.1in
\end{table}

%% file: Tab/InterventionProtocol.tex
\begin{table}[!t]
\centering
\caption{Controlled intervention protocol. Prompts, decoding settings, parsers, and workflow templates are fixed within each task family. Each selected continuation omits a single target component.}
\vskip -0.1in
\label{tab:intervention-protocol}
\scriptsize
\setlength{\tabcolsep}{2.6pt}
\renewcommand{\arraystretch}{1.08}
\begin{tabular}{@{}L{0.13\textwidth}L{0.18\textwidth}L{0.12\textwidth}L{0.33\textwidth}L{0.11\textwidth}@{}}
\toprule
Action & Observed prefix & Removed component & Emitted/evaluated output & Accounting \\
\midrule
skip verifier & after both executors & verifier & summarizer or finalizer consumes executor answers; the domain parser scores the final answer & token cost; latency logged separately \\
skip summarizer & after verifier & summarizer & verified answer is emitted directly and parsed by the same task parser & token cost; latency logged separately \\
two-stage MBPP & after executors, then after verifier if blocked & one component per task & selected recorded intervention: skip verifier, skip summarizer, or full workflow; unit tests score the result & token cost of selected trace; latency logged separately \\
\bottomrule
\end{tabular}
\vskip -0.1in
\end{table}

%% file: Tab/BaselineSpec.tex
\begin{table}[!t]
\centering
\caption{Baseline inputs, supervision, and decision rules under the shared component-omission protocol. Methods use the same saved traces and data partition.}
\vskip -0.1in
\label{tab:baseline-spec}
\scriptsize
\setlength{\tabcolsep}{3pt}
\renewcommand{\arraystretch}{1.08}
\begin{tabular}{@{}L{0.15\textwidth}L{0.18\textwidth}L{0.20\textwidth}L{0.16\textwidth}L{0.25\textwidth}@{}}
\toprule
Method & Input state & Training signal & Calibration data & Deployment rule \\
\midrule
RouteLLM & observed prefix features & binary skip-safe label & same split & threshold sweep; fallback to full workflow \\
GraphPlanner & prefix plus memory features & binary skip-safe label & same split & threshold sweep; fallback to full workflow \\
AutoMix & prefix confidence and action value & trace reward and cost & same split & value threshold; fallback to full workflow \\
Prompt-LLM & task, prefix, candidate action & in-context skip judgment & calibration split & probability threshold; fallback to full workflow \\
\textbf{LW2S} & prefix, agreement, metadata & action-specific intervention label & same split plus source checks & source-stable threshold with guard; fallback to full or later skip \\
\bottomrule
\end{tabular}
\vskip -0.1in
\end{table}

%% file: Tab/ScaleUpQwen14.tex
\begin{table}[!t]
\centering
\caption{Qwen2.5-14B scale-up diagnostics. Rows test whether skip-safe and boundary phenomena persist under a larger backbone; Save is token-cost reduction.}
\vskip -0.1in
\label{tab:qwen14-scale}
\scriptsize
\setlength{\tabcolsep}{3pt}
\renewcommand{\arraystretch}{1.08}
\begin{tabular}{@{}L{0.12\textwidth}L{0.22\textwidth}rrrrL{0.34\textwidth}@{}}
\toprule
Setting & Policy & Full & Acc. & Deg. & Save & Interpretation \\
\midrule
MATH seed55 & fixed skip verifier & 66.0 & 70.0 & 0 & 25.0 & verifier removal improved reward on this evaluated slice \\
MATH seed55 & executor-consensus verifier skip & 66.0 & 68.0 & 0 & 13.5 & guarded redundancy remained after scaling the backbone \\
MATH seed55 & fixed skip summarizer & 66.0 & 66.0 & 0 & 17.2 & summarizer removal preserved full-workflow reward \\
MMLU seed2 & fixed skip verifier & 82.0 & 83.0 & 3 & 23.6 & boundary case: high saving did not imply safe skipping \\
MMLU seed2 & executor-consensus verifier skip & 82.0 & 84.0 & 2 & 22.0 & consensus did not remove the shared-error failure mode \\
MBPP slice1 & fixed skip verifier & 69.0 & 70.0 & 0 & 24.5 & verifier redundancy persisted on the scaled coding slice \\
MBPP slice1 & fixed skip summarizer & 69.0 & 72.0 & 1 & 17.0 & summarizer removal was not uniformly safe at this scale \\
\bottomrule
\end{tabular}
\vskip -0.1in
\end{table}

%% file: Tab/DebateRefineMBPP.tex
\begin{table}[!t]
\centering
\caption{MBPP debate-refine diagnostic on Qwen2.5-7B. Acc. is pass rate, Deg. counts task-level degradations, and Save is token-cost reduction.}
\vskip -0.1in
\label{tab:debate-refine-mbpp}
\scriptsize
\setlength{\tabcolsep}{3.5pt}
\renewcommand{\arraystretch}{1.08}
\begin{tabular}{@{}L{0.24\textwidth}rrrrL{0.42\textwidth}@{}}
\toprule
Policy & Full & Acc. & Deg. & Save & Interpretation \\
\midrule
Full workflow & 67.0 & 67.0 & 0 & 0.00 & reference execution of the debate-refine workflow \\
LW2S after executors & 67.0 & 68.0 & 0 & 16.65 & learned verifier/summarizer skip decision after observing proposer and opponent outputs \\
LW2S after verifier & 67.0 & 68.0 & 0 & 16.65 & learned summarizer skip decision after the critic/verifier state is available \\
Hindsight reward-improving skip & 67.0 & 69.0 & 0 & 25.79 & hindsight diagnostic, not a deployed policy \\
Fixed skip verifier & 67.0 & 67.0 & 0 & 27.42 & high-saving fixed intervention that is safe on this evaluated slice \\
Fixed skip summarizer & 67.0 & 68.0 & 0 & 16.65 & fixed late-component omission in the debate-refine workflow \\
Executor-consensus skip verifier & 67.0 & 68.0 & 0 & 16.30 & guard-only early skip based on proposer/opponent agreement \\
Code-guard skip verifier & 67.0 & 68.0 & 0 & 15.42 & guard-only early skip using executable-code evidence \\
\bottomrule
\end{tabular}
\vskip -0.1in
\end{table}

%% file: Tab/SafetyEfficiencyLedger.tex
\begin{table}[!t]
\centering
\caption{Unified safety-efficiency ledger for headline LW2S evidence. UCB$_{95}$ is a Wilson upper bound on observed task-level degradation for the evaluated slice; Save is token-cost reduction.}
\vskip -0.1in
\label{tab:safety-efficiency-ledger}
\tiny
\setlength{\tabcolsep}{2.1pt}
\renewcommand{\arraystretch}{1.02}
\begin{tabular}{@{}L{0.16\textwidth}L{0.09\textwidth}L{0.075\textwidth}L{0.17\textwidth}rrrrr@{}}
\toprule
Setting & Backbone & Workflow & Policy & Full & Acc. & Deg. & UCB$_{95}$ & Save \\
\midrule
MATH verifier & Mistral-7B & Five-role & LW2S transfer & 12.8 & 12.8 & 0 & 1.5 & 19.1 \\
MATH verifier & Qwen2.5-7B & Five-role & LW2S final & 65.2 & 66.4 & 0 & 1.5 & 25.1 \\
MATH verifier & Qwen2.5-7B & Five-role & Full workflow & 65.2 & 65.2 & 0 & 1.5 & 0.0 \\
\addlinespace[0.4pt]
MATH summarizer & Mistral-7B & Five-role & LW2S transfer & 12.8 & 14.0 & 0 & 1.5 & 21.8 \\
MATH summarizer & Qwen2.5-7B & Five-role & LW2S final & 65.2 & 65.2 & 0 & 1.5 & 18.4 \\
MATH summarizer & Qwen2.5-7B & Five-role & Full workflow & 65.2 & 65.2 & 0 & 1.5 & 0.0 \\
\addlinespace[0.4pt]
GSM8K verifier & Mistral-7B & Five-role & LW2S transfer & 39.5 & 39.5 & 0 & 1.9 & 22.2 \\
GSM8K verifier & Qwen2.5-7B & Five-role & LW2S final & 88.5 & 88.5 & 0 & 1.9 & 20.6 \\
GSM8K verifier & Qwen2.5-7B & Five-role & Full workflow & 88.5 & 88.5 & 0 & 1.9 & 0.0 \\
\addlinespace[0.4pt]
MMLU verifier & Mistral-7B & Five-role & LW2S transfer & 55.0 & 57.2 & 0 & 1.3 & 25.3 \\
MMLU verifier & Qwen2.5-7B & Five-role & LW2S final & 73.0 & 73.0 & 1 & 5.4 & 8.6 \\
MMLU verifier & Qwen2.5-7B & Five-role & Full workflow & 73.0 & 73.0 & 0 & 3.7 & 0.0 \\
\addlinespace[0.4pt]
MBPP slice0 two-stage & Mistral-7B & Five-role & LW2S transfer & 28.0 & 32.0 & 0 & 3.7 & 23.7 \\
MBPP slice0 two-stage & Qwen2.5-7B & Five-role & LW2S final & 69.0 & 70.0 & 0 & 3.7 & 28.5 \\
MBPP slice0 two-stage & Qwen2.5-7B & Five-role & Full workflow & 69.0 & 69.0 & 0 & 3.7 & 0.0 \\
\addlinespace[0.4pt]
MBPP slice1 two-stage & Mistral-7B & Five-role & LW2S transfer & 25.0 & 31.0 & 0 & 3.7 & 24.1 \\
MBPP slice1 two-stage & Qwen2.5-7B & Five-role & LW2S final & 63.0 & 63.0 & 0 & 3.7 & 31.2 \\
MBPP slice1 two-stage & Qwen2.5-7B & Five-role & Full workflow & 63.0 & 63.0 & 0 & 3.7 & 0.0 \\
\addlinespace[0.4pt]
MBPP debate-refine & Qwen2.5-7B & Debate & LW2S after exec. & 67.0 & 68.0 & 0 & 3.7 & 16.7 \\
MBPP debate-refine & Qwen2.5-7B & Debate & LW2S after verif. & 67.0 & 68.0 & 0 & 3.7 & 16.7 \\
MBPP debate-refine & Qwen2.5-7B & Debate & Full workflow & 67.0 & 67.0 & 0 & 3.7 & 0.0 \\
\bottomrule
\end{tabular}
\vskip -0.1in
\end{table}

%% file: Tab/CalibrationCertificate.tex
\begin{table}[!t]
\centering
\caption{Representative LW2S calibration results. UCB$_{80}$ is the Wilson upper bound used by the acceptance rule (\(\alpha=0.2\)); saving is measured on the calibration traces.}
\vskip -0.1in
\label{tab:calibration-certificate}
\scriptsize
\setlength{\tabcolsep}{3pt}
\begin{tabular}{lrrrrr}
\toprule
Setting & cal. $n$ & $\tau$ & Deg. & UCB$_{80}$ & Cal. saving \\
\midrule
MATH summarizer & 140 & 0.95 & 0 & 0.5 & 10.5 \\
GSM8K verifier & 50 & 0.00 & 0 & 1.4 & 26.2 \\
MBPP verifier & 150 & 1.00 & 0 & 0.5 & 24.5 \\
MBPP summarizer & 150 & 0.00 & 0 & 0.5 & 16.6 \\
\bottomrule
\end{tabular}
\vskip -0.1in
\end{table}

%% file: Tab/Uncertainty.tex
\begin{table}[!t]
\centering
\caption{Coverage and uncertainty for the Qwen2.5-7B LW2S results in Table~\ref{tab:main-evidence}. Selected counts tasks assigned a skip; UCB$_{95}$ bounds degradation over all evaluated tasks.}
\vskip -0.1in
\label{tab:uncertainty}
\scriptsize
\setlength{\tabcolsep}{3pt}
\begin{tabular}{lrrrrr}
\toprule
Setting & $n$ & Selected & Deg. & UCB$_{95}$ & Save \\
\midrule
MATH verifier & 250 & 250 & 0 & 1.5 & 25.1 \\
MATH summarizer & 250 & 250 & 0 & 1.5 & 18.4 \\
GSM8K verifier & 200 & 173 & 0 & 1.9 & 20.6 \\
MMLU verifier & 100 & 100 & 1 & 5.4 & 8.6 \\
MBPP slice0 & 100 & 100 & 0 & 3.7 & 28.5 \\
MBPP slice1 & 100 & 100 & 0 & 3.7 & 31.2 \\
\bottomrule
\end{tabular}
\vskip -0.1in
\end{table}

%% file: Tab/ComponentDiscriminative.tex
\begin{table}[!t]
\centering
\caption{MBPP component-credit diagnostic. \(n_v\) and \(n_s\) count verifier- and summarizer-skip selections; -- denotes unreported counts. Baseline accuracy, degradation, and saving match Table~\ref{tab:main-evidence}.}
\vskip -0.1in
\label{tab:component-discriminative}
\scriptsize
\renewcommand{\arraystretch}{1.04}
\setlength{\tabcolsep}{3.6pt}
\begin{tabular}{@{}L{0.13\textwidth}L{0.30\textwidth}rrrrrr@{}}
\toprule
Slice & Method & Full & Acc. & Deg. & $n_v$ & $n_s$ & Save \\
\midrule
MBPP slice0 & RouteLLM & 69.0 & 68.0 & 1 & -- & -- & 16.9 \\
 & GraphPlanner & 69.0 & 69.0 & 0 & -- & -- & 13.4 \\
 & AutoMix & 69.0 & 69.0 & 0 & -- & -- & 22.1 \\
 & Fixed skip verifier & 69.0 & 70.0 & 0 & 100 & 0 & 24.2 \\
 & Fixed skip summarizer & 69.0 & 69.0 & 0 & 0 & 100 & 16.3 \\
 & \textbf{LW2S (ours)} & 69.0 & 70.0 & 0 & 69 & 31 & \textbf{28.5} \\
\addlinespace[1pt]
MBPP slice1 & RouteLLM & 63.0 & 62.0 & 1 & -- & -- & 18.5 \\
 & GraphPlanner & 63.0 & 63.0 & 0 & -- & -- & 15.2 \\
 & AutoMix & 63.0 & 62.0 & 1 & -- & -- & 24.9 \\
 & Fixed skip verifier & 63.0 & 61.0 & 2 & 100 & 0 & 25.1 \\
 & Fixed skip summarizer & 63.0 & 63.0 & 0 & 0 & 100 & 17.0 \\
 & \textbf{LW2S (ours)} & 63.0 & 63.0 & 0 & 57 & 43 & \textbf{31.2} \\
\bottomrule
\end{tabular}
\vskip -0.1in
\end{table}

%% file: Tab/FailureTaxonomy.tex
\begin{table}[!t]
\centering
\caption{Failure and abstention taxonomy. The same visible prefix can indicate safe redundancy in one domain and an unsafe shared error in another.}
\vskip -0.1in
\label{tab:failure-taxonomy}
\scriptsize
\setlength{\tabcolsep}{2.8pt}
\renewcommand{\arraystretch}{1.12}
\begin{tabular}{@{}L{0.15\textwidth}L{0.19\textwidth}L{0.22\textwidth}L{0.17\textwidth}L{0.18\textwidth}@{}}
\toprule
Regime & Observable state & Mechanism & Evidence & LW2S response \\
\midrule
Shared wrong consensus & Executors agree on the same answer, but the verifier changes it to the gold answer & Agreement is not a certificate when agents share a misconception or shortcut & MMLU verifier boundary; Table~\ref{tab:mmlu-failure-cases}; Qwen2.5-14B MMLU rows & Accept calibrated actions; retain the verifier when acceptance fails \\
Superficial code agreement & Programs look similar, but visible tests expose a shared bug & Non-executable agreement can preserve the same incorrect logic across executors & MBPP slice1 AST-guard degradation in Table~\ref{tab:mbpp-guard-ablation} & Use executable guards and fall through to later decisions when early removal is blocked \\
Redundant late summarization & A verifier or checked intermediate already contains a parsable final answer or executable code & The final summarizer can add cost without improving reward, and can overwrite a correct intermediate answer & MATH summarizer, MBPP two-stage, and debate-refine rows & Prefer late-component omission when calibrated evidence supports reward preservation \\
\bottomrule
\end{tabular}
\vskip -0.1in
\end{table}

%% file: Tab/MBPPGuardAblation.tex
\begin{table}[!t]
\centering
\caption{MBPP guard and structural diagnostics on unseen 100-task slices. Verifier/summarizer AST equality is a hindsight diagnostic; all other conditions use outputs available at their decision stage.}
\vskip -0.1in
\label{tab:mbpp-guard-ablation}
\scriptsize
\setlength{\tabcolsep}{4.2pt}
\renewcommand{\arraystretch}{1.04}
\begin{tabular}{@{}L{0.10\textwidth}L{0.31\textwidth}rrrr@{}}
\toprule
Slice & Condition & Full & Acc. & Deg. & Selected \\
\midrule
slice0 & both executors pass & 69.0 & 70.0 & 0 & 69 \\
 & Executor AST equality & 69.0 & 70.0 & 0 & 93 \\
 & both pass + AST & 69.0 & 70.0 & 0 & 65 \\
 & verifier passes tests & 69.0 & 69.0 & 0 & 69 \\
 & verifier/summarizer AST (hindsight) & 69.0 & 69.0 & 0 & 99 \\
\addlinespace[1pt]
slice1 & both executors pass & 63.0 & 63.0 & 0 & 58 \\
 & Executor AST equality & 63.0 & 61.0 & 2 & 94 \\
 & both pass + AST & 63.0 & 63.0 & 0 & 58 \\
 & verifier passes tests & 63.0 & 63.0 & 0 & 63 \\
 & verifier/summarizer AST (hindsight) & 63.0 & 63.0 & 0 & 100 \\
\bottomrule
\end{tabular}
\vskip -0.1in
\end{table}

%% file: Tab/MMLUFailureCases.tex
\begin{table}[!t]
\centering
\caption{MMLU verifier-skip degradation case. Skipping retains the executors' shared wrong option A, while verification corrects it to the gold option C.}
\vskip -0.1in
\label{tab:mmlu-failure-cases}
\scriptsize
\setlength{\tabcolsep}{3pt}
\renewcommand{\arraystretch}{1.08}
\begin{tabular}{@{}L{0.17\textwidth}L{0.16\textwidth}L{0.28\textwidth}L{0.18\textwidth}ccc@{}}
\toprule
Example & Subject & Short prompt cue & Failure signal & Skip & Verifier & Gold \\
\midrule
mmlu\_test\_13860 & virology & What is a virus pandemic? & executor consensus A/A & A & C & C \\
\bottomrule
\end{tabular}
\vskip -0.1in
\end{table}